\documentclass[11pt]{article}
\usepackage{amsmath}
\usepackage{jmlr2e}
\usepackage{amssymb}
\usepackage{graphicx}
\usepackage{geometry}
\usepackage{algorithm}

\usepackage{hyperref}
\usepackage{xcolor,colortbl}
\usepackage{soul}

\newtheorem{claim}{Claim}

\newcommand{\dom}{dom}
\newcommand{\VCdim}{VCdim}

\newcommand{\str}{s}
\newcommand{\s}{s}
\newcommand{\R}{\mathbb{R}}
\newcommand{\bC}{\bar{C}}

\newcommand{\sstr}{st}

\newcommand{\proj}{|}
\newcommand{\one}{{\mathbf{1}}}
\newcommand{\zero}{{\mathbf{0}}}
\newcommand{\powset}{\mathcal{P}}
\newcommand{\size}{size}

\renewcommand{\st}{\operatorname{st}}
\renewcommand{\ss}{\operatorname{st}}

\definecolor{Gray}{gray}{0.85}

\title{
Labeled compression schemes for extremal classes\thanks{Supported by NSF grant IIS-1118028.}
}

\begin{document}
\title{Labeled compression schemes for extremal classes}
 \author{\name{Shay Moran} \email{shaymoran1@gmail.com}\\
 \addr Technion, Israel Institute of Technology, Haifa 32000, Israel,
Microsoft Research, Herzliya, and
Max Planck Institute for Informatics, Saarbr\"{u}cken, Germany.
 \AND
 \name{Manfred K. Warmuth} \email{manfred@cse.ucsc.edu}\\
 \addr University of Santa Cruz, California
 }

\editor{}
\maketitle

\begin{abstract}
It is a long-standing open problem whether
there always exists a compression scheme whose
size is of the order of the Vapnik-Chervonienkis (VC) dimension $d$.
Recently compression schemes of size exponential in $d$
have been found for any concept class of VC dimension $d$. 
Previously, compression schemes of size $d$ have
been given for maximum classes, which are
special concept classes whose size equals an upper bound due to Sauer-Shelah.
We consider a generalization of maximum classes
called extremal classes. Their definition is based on a powerful generalization of the Sauer-Shelah
bound called the Sandwich Theorem, which has been studied in several areas of combinatorics
and computer science.
The key result of the paper is a construction of a sample
compression scheme for extremal classes of size equal to
their VC dimension.
We also give a number of open problems concerning the
combinatorial structure of extremal classes and the existence of
unlabeled compression schemes for them.
\end{abstract}

\section{Introduction}
Generalization and compression/simplification are two basic facets of
``learning''.
Generalization concerns the expansion of existing knowledge
and compression concerns simplifying our explanations of it.
In machine learning, compression and generalization are deeply related:
learning algorithms perform compression and the ability to
compress guarantees good generalization.

A simple form of this connection is how Occam's Razor \citep{BlumerEHW87} is
manifested in Machine Learning: if the input sample can be
compressed to a small number of bits which encodes a
hypothesis consistent with the input sample, then good
generalization of this hypothesis is guaranteed.
A more sophisticated notion of compression is given by
"sample compression schemes" \citep{littleWarm}. In these schemes
the input sample is compressed to a carefully chosen small subsample that encodes
a hypothesis consistent with the input sample.
For example support vector machine can be seen as
compressing the original sample to the subset of support vectors
which represent a maximum margin hyperplane that is
consistent with the entire original sample.

What is the connection to generalization?
In the Occam's razor setting, the generalization error decreases with the number of bits that are used 
to encode the output hypothesis.
Similarly for compression
schemes, the generalization error decreases with the sample size.

In the learning model considered here, the learner is given a sample
consistent with an unknown concept from a target concept class. 
From the given sample, the learner aims to construct a hypothesis
that yields a good generalization i.e.~a good approximation of the
unknown concept. A core question is what parameter of the
concept class characterizes the sample size required
for good generalization? 
The Vapnik-Chervonenkis (VC) dimension serves as such a parameter
\citep{zbMATH04143473} where the exact definition of generalization
underlying our discussion
is specified by the Probably Approximately Correct (PAC)
model of learning \citep{zbMATH03943062,VC1}.
We believe that the size of the best compression scheme is an
alternate parameter and has several additional advantages:
\begin{itemize}
\item 
Compression schemes frame many natural algorithms
(e.g.~support vector machines). 
This gives sample compression schemes
a constructive flavor.
\item
Unlike the VC dimension, the definition of sample
compression schemes as well as the fact that they yield
low generalization error extends naturally to
multi label concept classes~\citep{DBLP:conf/alt/SameiYZ14}. 
This is particularly interesting when
the number of labels is very large (or possibly infinite),
because for that case there is no known
combinatorial parameter that characterizes 
the sample complexity in the PAC model~(See \citep{DBLP:conf/colt/DanielyS14}).
The size of the best sample compression scheme is therefore a natural
candidate for a universal parameter that characterizes the sample complexity
in the PAC model.
\end{itemize}

\subsection*{Previous work}

\cite{littleWarm} defined \emph{sample compression schemes} 
and showed that in the PAC model of learning, the 
sample size required for learning grows linearly with the size
of the subsamples the scheme compresses to.
They have also posed the other direction as an open question: 
Does every concept class have a compression scheme of size 
depending only on its VC dimension?
Later \cite{DBLP:journals/ml/FloydW95} and \cite{DBLP:conf/colt/Warmuth03}, 
refined this question: Does every class of VC dimension $d$
have a sample compression scheme of size $O(d)$.

\cite{DBLP:journals/dam/Ben-DavidL98} 
proved a compactness theorem for sample compression schemes.
It essentially says that existence of compression schemes for infinite 
classes follows\footnote{The proof of that theorem is however non-constructive.} 
from the existence of such schemes for finite classes.
Thus, it suffices to consider only finite concept classes. 
\cite{DBLP:journals/ml/FloydW95} constructed sample compression schemes 
of size $\log|C|$ for every concept class $C$. 
More recently~\cite{MSWY15}
have constructed sample compression schemes of size 
$\exp(d)\log\log|C|$ where $d=\VCdim(C)$. Finally,~\cite{DBLP:journals/corr/MoranY15}
have constructed sample compression scheme of size $\exp(d)$, 
resolving Littlestone and Warmuth's question.
Their compression scheme is based on an earlier compression scheme by 
\cite{DBLP:journals/iandc/Freund95,schapire2012boosting}
which was defined in the context of boosting. This sample compression scheme is of variable size:
It compresses samples of size $m$ to subsamples of size 
$O(d\log m)$.

For many natural and important families of concept classes, sample compression
schemes of size equal the VC dimension were constructed, 
revealing connections between sample compression schemes
and other fields such as combinatorics, geometry, model
theory, and algebraic topology
(e.g.~\cite{DBLP:conf/colt/Floyd89,DBLP:journals/siamcomp/HelmboldSW92,DBLP:journals/dam/Ben-DavidL98,chernikovS,DBLP:journals/jcss/RubinsteinBR09,RR3,DBLP:conf/colt/LivniS13}).
Despite this rich body of work, the refined question whether there exists a compression scheme whose size is equal or linear in the VC dimension remains open. 

\cite{DBLP:journals/ml/FloydW95}
observed that in order to prove the conjecture it suffices
to consider only maximal classes (A class $C$ is maximal if no
concept can be added without increasing the VC dimension).
Furthermore, they constructed sample compression schemes of size $d$ for every 
{\em maximum class} of VC dimension $d$. 
These classes are maximum in the sense that their size equals
an upper bound (due to Sauer-Shelah) on the size of any concept class of VC dimension $d$. 
Later, \cite{DBLP:journals/jmlr/KuzminW07} and \cite{RR3} provided even more
efficient and combinatorially elegant sample compression schemes for maximum classes 
that are called unlabeled compression schemes because the
labels of the subsample are not needed to encode the output hypothesis.

One possibility of making a progress
on Floyd and Warmuth's question is by extending the optimal
compression schemes for maximum classes to a more general family.
In this paper we consider a natural and rich generalization of maximum classes 
which are known by the name extremal classes (or shattering extremal classes). 
Similar to maximum classes, these classes are defined when a
certain inequality known as The Sandwich Theorem is tight. 
This inequality generalizaes the
Sauer-Shelah bound.
The Sandwich Theorem as well as extremal classes were discovered several times and independently by several groups of researchers
and in several contexts such as Functional analysis~\citep{Pajor}, Discrete-geometry~\citep{Law}, Phylogenetic Combinatorics~\citep{Dress1,Dress2} and Extremal Combinatorics~\citep{BR89,BR95}.
Even though a lot of knowledge regarding the structure of extremal classes
has been accumulated,
the understanding of these classes is still considered incomplete 
by several authors~\citep{BR95,Greco98,S-ext}. 

\subsection*{Our results}

Our main result is a construction of sample compression scheme of size $d$ 
for every extremal class of VC dimension $d$. 
When the concept class is maximum, then our scheme
specializes to the compression scheme for maximum
classes given in \citep{DBLP:journals/ml/FloydW95}.
Our generalized sample compression scheme for extremal
classes is still easy to describe.
However its analysis requires more
combinatorics and heavily exploits the rich structure of extremal classes.
Despite being more general, the construction is simple.
We also give explicit examples of maximal classes that are
extremal but not maximum (see Example~\ref{ex:26}). This means that the compression
scheme presented here is not implied by the previous
sample compression schemes for maximum classes, not even implicitly.

We also discuss a certain greedy peeling method for
producing an unlabeled compressions scheme. 
Such schemes were first conjectured in \citep{DBLP:journals/jmlr/KuzminW07}
and later proven to exist for maximum classes \citep{RR3}.
However the existence of such schemes for extremal classes
remains open. We relate the existence of such schemes to basic open
questions concerning the combinatorial structure of extremal classes.

\subsection*{Organization}

In Section~\ref{sec:extremal} we give some preliminary definitions and define extremal
classes. We also discuss some basic properties and give some examples of extremal classes which demonstrate their generality over maximum classes. 
In Section~\ref{sec:labeled}
we give a labeled compression scheme for any extremal class of VC dimension $d$. Finally, in Section~\ref{sec:unlabeled} we relate unlabeled compression schemes for extremal classes with basic open questions concerning extremal classes.

\section{Extremal Classes}
\label{sec:extremal}
\subsection{Preliminaries}

\paragraph{Concepts, concept classes, and the one-inclusion
graph.}
A concept $c$ is a mapping from some domain to $\{0,1\}$.
We assume for the sake of simplicity that the domain of $c$
(denoted by $\dom(c)$) is finite and allow the case that $\dom(c)=\emptyset$.
A concept $c$ can also be viewed as a characteristic function of
a subset of $\dom(c)$, i.e for any domain point $x\in \dom(c)$,
$c(x) = 1$ iff $x \in c$.
A concept class $C$ is a set of concepts with
the same domain (denoted by $\dom(C)$).
A concept class can be represented
by a binary table (see Fig.\ \ref{fig:example}),
where the rows correspond to concepts and the columns to the elements of $\dom(C)$.
Whenever the elements in $\dom(C)$ are clear from the context, then we
represent concepts as bit strings of length $|\dom(C)|$
(See Fig.\ \ref{fig:example}).

The concept class $C$ can also be represented as a 
subgraph of the Boolean hypercube with 
$\lvert \dom(C)\rvert$ dimensions.
Each dimension corresponds to a particular domain element,
the vertices are the concepts in $C$ 
and two concepts are connected with an edge if they
disagree on the label of a single element (Hamming distance 1).
This graph is called the {\em one-inclusion graph}
of $C$ \citep{DBLP:journals/iandc/HausslerLW94}. Note that each edge is naturally
labeled by the single dimension/element
on which the incident concepts disagree
(See Fig.\ \ref{fig:example}). 
\begin{figure}[t]
\begin{tabular}{ll}
\begin{minipage}[h]{.24\textwidth}
\begin{tabular}{r|cc cc cc}
&$x_1$&$x_2$&$x_3$&$x_4$&$x_5$&$x_6$\\
\hline\hline
$c_1$&0&0&0&0&0&0\\

$c_{2}$&0&0&1&0&0&0\\
$c_{3}$&0&1&0&0&0&0\\
$c_{4}$&1&0&0&0&0&0\\

$c_{5}$&0&0&1&0&1&0\\
$c_{6}$&0&0&1&1&0&0\\
$c_{7}$&1&0&1&0&0&0\\
$c_{8}$&1&1&0&0&0&0\\

$c_{9}$&0&0&1&0&1&1\\
$c_{10}$&0&0&1&1&1&0\\
$c_{11}$&0&0&1&1&0&1\\
$c_{12}$&1&0&1&1&0&0\\
$c_{13}$&1&1&1&0&0&0\\
$c_{14}$&1&1&0&1&0&0\\

$c_{15}$&0&0&1&1&1&1\\
$c_{16}$&1&0&1&1&0&1\\
$c_{17}$&1&1&1&1&0&0\\
$c_{18}$&1&0&1&1&1&1
\end{tabular}
\end{minipage}
&
\begin{minipage}[h]{0.85\textwidth}
\vspace{-1.2cm}
\begin{center}
\includegraphics[width=.8\textwidth]{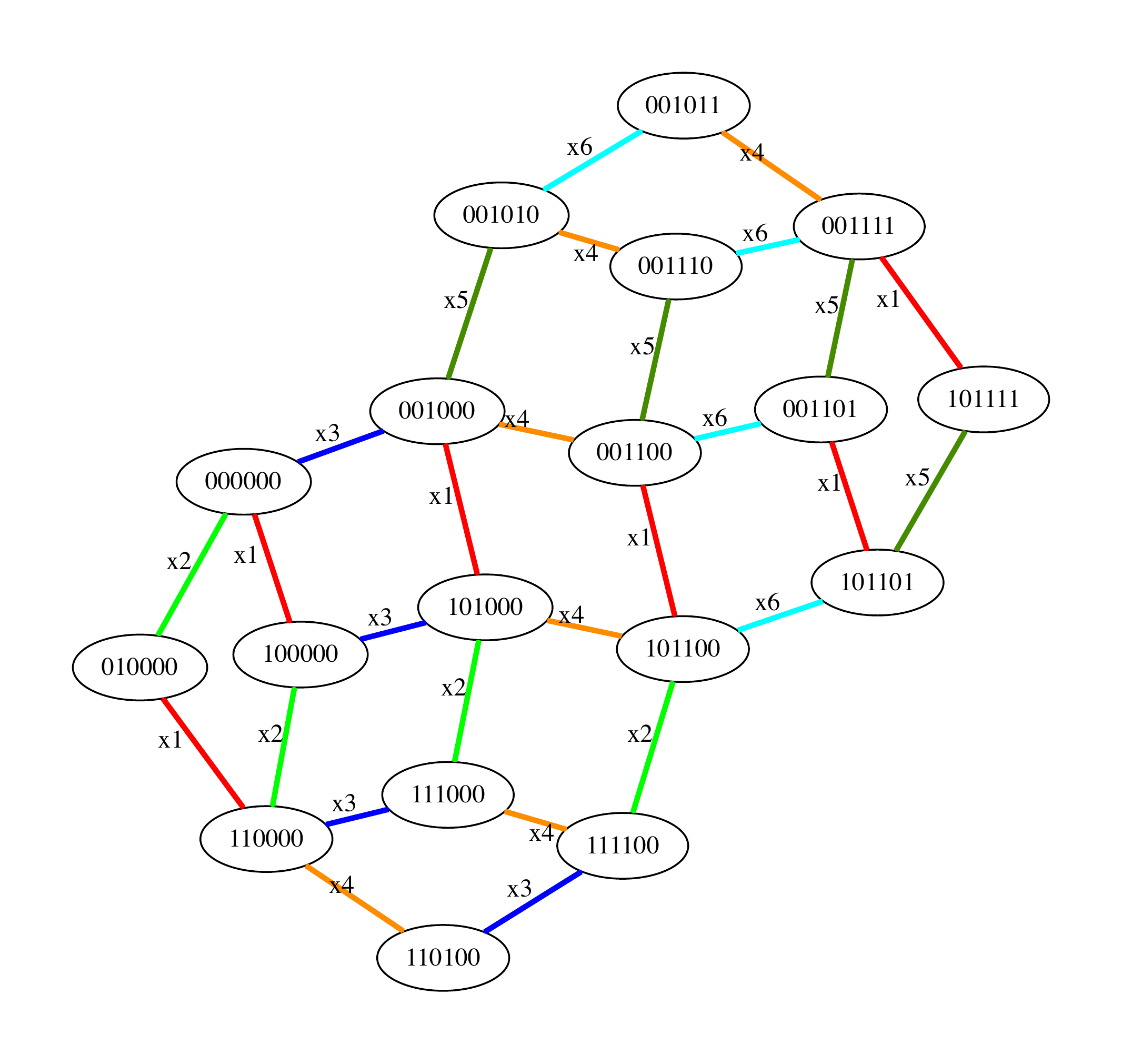}
\end{center}
\vspace{-1cm}
\end{minipage}
\end{tabular}
\caption{Table and one-inclusion graph of an extremal class
$C$ of VC dimension 2. The reduction
$C^{x_2}=\{00000,10000,11000,11100\}$ has the domain
$\{x_1,x_3,x_4,x_5,x_6\}$. Note that each
concept in $C^{x_2}$ corresponds to an edge labeled $x_2$.
Similarly $C^{\{x_3,x_4\}}$ consists of the single concept $\{1100\}$
over the domain $\{x_1,x_2,x_5,x_6\}$. Note that this
concept corresponds to the single cube of $C$ with
dimension set $\{x_3,x_4\}$.}
\label{fig:example}
\end{figure}

\paragraph{Restrictions and samples.}         
We denote the \textit{restriction/sample} of a concept $c$ 
onto $S \subseteq \dom(c)$ as $c|S$. 
This concept has the restricted domain $S$ 
and labels this domain consistently with $c$.
Essentially concept $c|S$ is obtained by removing 
from row $c$ in the table all columns not in $S$.
The {\em restriction/set of samples} of an entire class $C$ onto $S\subseteq
\dom(C)$ is denoted as $C|S$.
A table for $C|S$ is produced by simply removing all columns not in $S$ 
from the table for $C$ and collapsing identical rows.%
\footnote
{We define $c|\emptyset = \emptyset$. Note that $C|\emptyset = \{\emptyset\}$ if 
$C \neq \emptyset$ and $\emptyset$ otherwise.}
Also the one-inclusion graph for the restriction $C|S$ 
is now a subgraph of the Boolean hypercube with $|S|$  dimensions
instead of the full dimension $|\dom(C)|$.
We also use $C-S$ as shorthand for $C|(\dom(C)\setminus S)$ 
(since the columns labeled with $S$ are removed from the table). 
Note that the sub domain $S\subseteq \dom(C)$ induces an equivalence class
on $C$: Two concepts $c,c'\in C$ are
equivalent iff $c|S=c'|S$. Thus there is one equivalence class
per concept of $C|S$.

\paragraph{Cubes.}
A concept class $B$ is called a cube 
if for some subset $S$ of the domain $\dom(B)$,
the restriction $B|S$ is the set of all $2^{|S|}$ concepts 
over the domain $S$ and the class $B-S$ contains a single concept. We denote this
single concept by $\mathrm{tag}(B)$. 
In this case, we say that $S$ is the dimension set of $B$
(denoted as $\dim(B)$).
For example, if $B$ contains two concepts that are incident
to an edge labeled $x$ then $B$ is a cube with $\dim(B)=\{x\}$.
We say that $B$ {\em is a cube of} concept class $C$ 
if $B$ is a cube that is a subset of $C$.
We say that $B$ is a maximal cube of $C$ if there exists no
other cube of $C$ which strictly contains $B$.
When the dimensions are clear from the context, then a
concept is described as a bit string of length $\dom(C)$.
Similarly a cube, $B$, is described as an expression in
$\{0,1,*\}^{|\dom(C)}$, where the dimensions of $\dim(B)$
are the *'s and the remaining bits is the concept ${\mathrm tag}(B)$.

\paragraph{Reductions.}
In addition to the restriction it is common to define a second
operation on concept classes. We will describe this
operation using cubes. 
The \textit{reduction} $C^S$ is a concept class
on the domain $\dom(C)\setminus S$ which has one concept
per cube with dimensions set $S$
$$C^S:=\{\mathrm{tag}(B) : B\mbox{ is a cube of $C$ such that } \dim(B)=S\}.$$
The reduction with respect to a single dimension $x$ is denoted
as $C^x$. See Fig.\ \ref{fig:example} for some examples.

\paragraph{Shattering and strong shattering.}         
There are two important properties associated with subsets
$S$ of the domain of $C$. We say that $S\subseteq \dom(C)$ is {\em
shattered} by $C$, if $C|S$ is the set of all $2^{|S|}$ concepts 
over the domain $S$.
Furthermore, $S$ is {\em strongly shattered} by $C$, 
if $C$ has a cube with dimensions set
$S$. We use $\s(C)$ to denote all shattered sets of $C$
and $\st(C)$ to denote all strongly shattered sets,
respectively.
Clearly, both $\str(C)$ and $\sstr(C)$ are closed under the subset relation, 
and $\sstr(C)\subseteq\str(C)$.
         
The following theorem is the result of accumulated work 
by different authors, and parts of it were rediscovered 
independently several times \citep{Pajor,BR95,Dress1,ShatNews}. 
\vspace{-.05in}
\begin{theorem}[Sandwich Theorem]
\label{thm:sandwich}
Let $C$ be a concept class.
\vspace{-0.09in}
$$\lvert \st(C)\rvert \leq \lvert C\rvert \leq \lvert \s(C)\rvert.$$	
\end{theorem}
This theorem has been discovered independently several
times and has several proofs
(see~\citep{DBLP:journals/corr/abs-1211-2980} for more
details). One approach, which is also used in proving the
Sauer-Shelah Lemma~\citep{sauer,shelah}, is via
\emph{down-shifting}. We now sketch this
approach. In a down-shifting step we pick
a dimension $x\in\dom(C)$, and every $c\in C$ is replaced by its $x$-neighbor 
(i.e. the concept $c'$ which disagrees with $c$ only on $x$)
if the following conditions hold: (i) $c(x)=1$, and (ii) the $x$-neighbour
of $c$ does not belong to $C$.
One can easily verify that if $C'$ is obtained from $C$ by a down-shifting step, then $|C'|=|C|$, $\str(C')\subseteq\str(C)$,
and $\sstr(C')\supseteq\sstr(C)$.
Eventually, after enough down-shifting steps have been 
performed\footnote{In fact one step on each $x\in X$
suffices~\citep{DBLP:journals/corr/abs-1211-2980}.}
the resulting class becomes downward-closed (see Example \ref{ex:down} below). For such classes
the cardinality of $\str(C),\sstr(C),$ and $|C|$ are all equal. This implies
the inequalities in the Sandwich Theorem for the original class.
         
The inequalities in this theorem can be strict: Let $C\subseteq\{0,1\}^n$ 
be such that $C$ contains all boolean vectors with an even number of $1's$.
Then $\sstr(C)$ contains only the empty set and
$\str(C)$ contains all subsets of $\{1,\ldots,n\}$ of size at most $n-1$.
Thus in this example, $|\sstr(C)|=1$, $|C|=2^{n-1}$, and $|\str(C)|=2^{n}-1$.
         
The \emph{VC dimension}~\citep{VC1,zbMATH04143473} is defined as: 
\vspace{-0.1in}
$$
\VCdim(C) = \max\{|S|:S\in \str(C)\}.
$$
         
Note that by the definition of the VC-dimension:
\vspace{-0.08in}
$$\str(C)\subseteq\{S\subseteq \dom(C)~:~\lvert S\rvert\leq \VCdim(C) \}.$$
Hence, an easy consequence of Theorem \ref{thm:sandwich} is
that for every concept class $C$, 
we have $|C|\leq \sum_{i=0}^{\VCdim(C)}{{|\dom(C)|}\choose{i}}$. This is the well-known Sauer-Shelah Lemma~\citep{sauer,shelah}.
\subsection{Definition of extremal classes and examples}
Maximum classes are defined as concept classes which satisfy the Sauer-Shelah inequality with equality.
Analogously, \emph{extremal classes} are defined as concept
classes which satisfy the inequalities\footnote{There are
two inequalities in the Sandwich Theorem, but every class
which satisfies one of them with equality also satisfies
the other with equality (See Theorem~\ref{thm:extrchar}).} in the Sandwich Theorem with equality:
A concept class $C$ is \emph{extremal} if 
for every shattered set $S$ of $C$
there is a cube of $C$ with dimension set $S$, i.e.  $\s(C)=\st(C).$
Note that complementing the bits in a column of the table representing
$C$ does not affect the sets $\s(C)$, $\st(C)$ and
extremality is preserved. Also in the one inclusion graph, only 
the labels of the vertices are affected by such column complementations.

Every maximum class is an extremal class. 
Moreover, maximum classes of VC dimension $d$ are precisely the extremal classes 
for which the shattered sets consist of 
all subsets of the domain of size up to $d$.
The other direction does not hold - there are extremal classes that are not maximum.
All the following examples are extremal but not maximum.
\begin{example}
Consider the concept class $C$ over the domain
$\{x_1,\ldots,x_6\}$ given in Fig.\ \ref{fig:example}.
In this example
\begin{align*}
\sstr(C)=\str(C)= &\big\{\emptyset,\{x_1\},\{x_2\},\{x_3\},\{x_4\},\{x_5\},\{x_6\},\\
			  &\{x_1,x_2\},\{x_1,x_3\},\{x_1,x_4\},\{x_1,x_5\},\{x_1,x_6\},\{x_2,x_3\},\\
			  &\{x_2,x_4\},\{x_3,x_4\},\{x_4,x_5\},\{x_4,x_6\},\{x_5,x_6\}\big\}.
\end{align*}
This example also demonstrates the cubical structure of extremal classes.
\end{example}
\begin{example}
{\bf (Downward-closed classes)}\\
A standard example of a maximum class of VC dimension $d$
is $$C=\{c\in\{0,1\}^n: \mbox{ the number of $1$'s in $c$ is at most $d$}\}.$$
This is simply the hamming ball of radius $d$ around the all $0$'s concept.
A natural generalization of such classes are downward closed classes. 
We say that $C$ is downward closed
if for all $c\in C$ and for all $c'\leq c$, also $c'\in C$.
Here $c'\leq c$ means that for every $x\in \dom(C)$, $c'(x)\leq c(x)$.
It is not hard to verify that every downward closed class is extremal.
\label{ex:down}
\end{example}
\begin{example}
{\bf (Hyper-planes arrangements in a convex domain)} \\
Another standard set of examples for maximum classes comes
from geometry (see e.g.~\citep{Welzl}).
Let $H$ be an arrangement of hyperplanes in $\mathbb{R}^{d}$. 
For each hyperplane $p_i\in H$, pick one of half-planes determined
by $p_i$ to be its {\it positive side} and the other its {\it negative side}. 
The hyperplanes of $H$ cut $\R^d$ into open regions ({\it cells}). 
Each cell defines a binary mapping with domain $H$:
$$c(p_i)=\begin{cases}
1 & \mbox{if }c\mbox{ is in the positive side of }p_i\\
0 & \mbox{if }c\mbox{ is in the negative side of }p_i\,.
\end{cases}$$
It is known that if the hyperplanes are in general
position, then the set $C$ of all cells is a maximum class of VC dimension $d$.
         
Consider the following generalization of these classes:
Let $K\subseteq\mathbb{R}^d$ be a convex set. Instead of taking the vectors corresponding to 
all of the cells, take only those that correspond to cells that intersect $K$:
$$C_K =\{c : c~ \mbox{ corresponds to a cell that intersects }K\}.$$
$C_K$ is extremal. In fact, for $C_K$ to be extremal it is not even required that the
hyperplanes are in general position. It suffices to require that no $d+1$ hyperplanes
have a non-empty intersection (e.g. parallel hyperplanes are allowed).
Fig.\ \ref{fig:hyper} illustrates such a class $C_K$ in the plane.
These classes were studied in~\citep{DBLP:journals/corr/abs-1211-2980}.
\end{example}
\begin{figure}
\begin{center}
\includegraphics[width=.375\textwidth,height=.2\textheight]{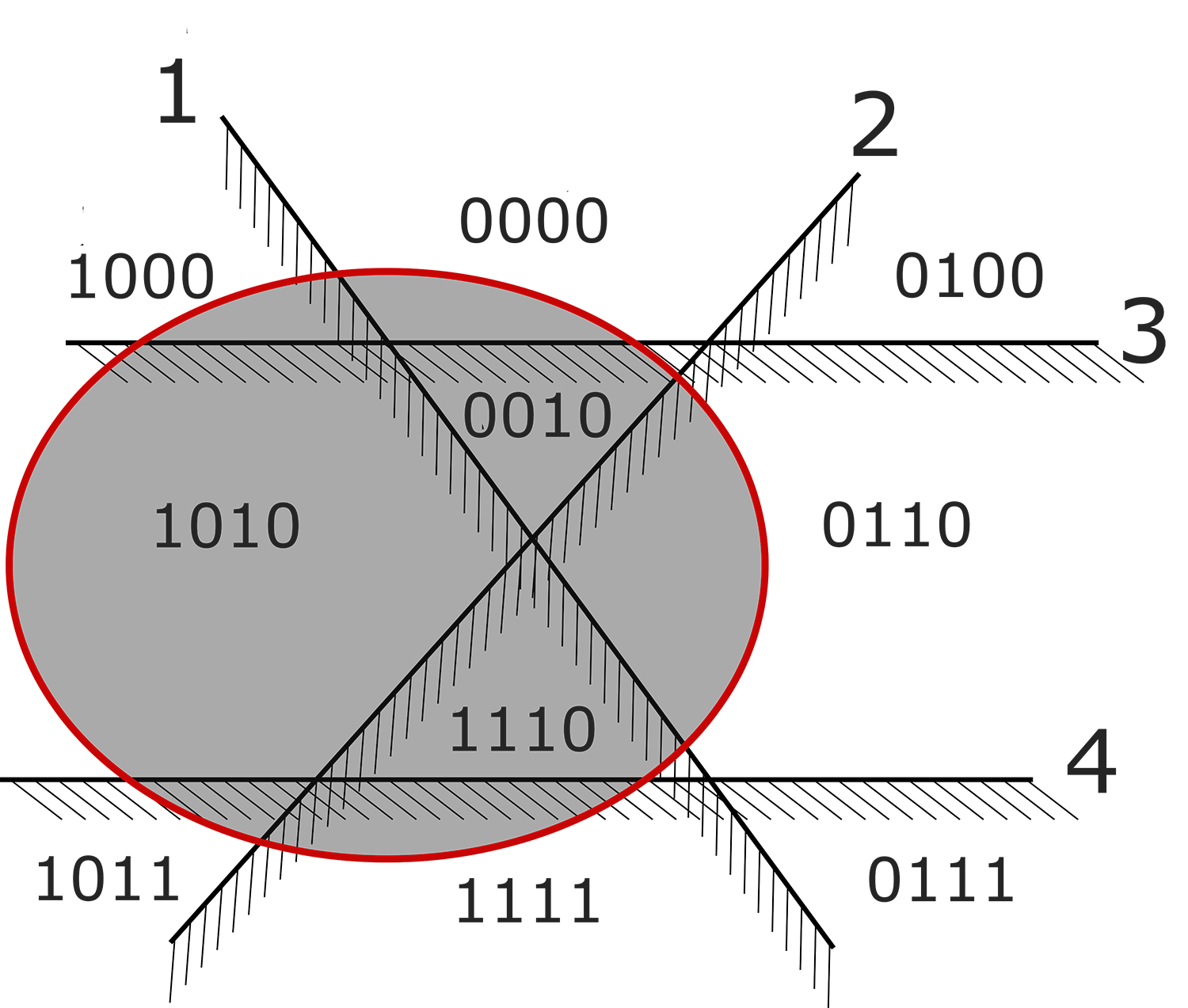}
\end{center}
\caption{An extremal class that correspond to the cells of a 
hyperplane arrangement of a convex set.
An arrangement of $4$ lines is given which partitions the
plane to $10$ cells.
Each cell corresponds to a binary vector which specifies its location relative to the lines.
For example the cell corresponding to $1010$ is on the positive sides of lines $1$ and $3$ and
on the negative side of lines $3$ and $4$.
Here the convex set $K$
is an ellipse and the extremal concept class consisting of
the cells the ellipse intersects is $C_K=\{1000,1010,1011,1111,1110,0010,0000,0110\}$
(the cells $0100,0111$ are not intersected by the ellipse).
The class $C_K$ here has VC dimension $2$. Note that it's shattered sets of size
$2$ are exactly the pairs of lines whose intersection point lies in the ellipse $K$.}
\label{fig:hyper}
\end{figure}

Interestingly, extremal classes also arise in the
context of graph theory:
\begin{example}\label{ex:25}
{\bf (Edge-orientations which preserve connectivity~\citep{DBLP:journals/combinatorics/KozmaM13})}\\
Let $G=(V,E)$ be an undirected simple graph
and let $\overrightarrow{E}$ be a fixed reference orientation.
Now an arbitrary orientation of $E$ is a function $d:E\rightarrow\{0,1\}$: 
If $d(e)=0$ then $e$ is oriented as in $\overrightarrow{E}$
and if $d(e)=1$ then $e$ is oriented opposite to $\overrightarrow{E}$.
Now let $s,t\in V$ be two fixed vertices, 
and consider all orientations of $E$ for which there exists a directed path
from $s$ to $t$. The corresponding class of orientations
$E\rightarrow \{0,1\}$ is 
an extremal concept class over the domain $E$. 
\end{example}
Moreover, the extremality of this class yields the following result in graph theory:
The number of orientations for which there exists a
directed path from $s$ to $t$ equals the number of
subgraphs for which there exists an undirected path from
$s$ to $t$. For a more thorough discussion and other
examples of extremal classes related to graph orientations 
see~\citep{DBLP:journals/combinatorics/KozmaM13}.
\begin{example}\label{ex:26} {\bf (A general construction of a maximal class that is
extremal but not maximum)}\\
Take a $k$-dimensional cube and glue to each of its vertices an edge of a new distinct dimension.
The resulting class has $2^{k+1}$ concepts and $n=2^k+k$ dimensions.
Let $C$ be the complement of that class.
\begin{claim}\label{lem:maximalextremal}
$C$ is an extremal maximal class of VC dimension $n-2$ which is not maximum.
\end{claim}
We prove this claim in Appendix~\ref{app:maximalextremal}.
Note that $|C|=2^n-2^{k+1}=2^{2^k+k}-2^{k+1}$ and maximum
classes of VCdim $d=n-2$ over $n$ dimensions 
have size $2^n-n-1=2^{2^k+k}-2^k-k-1$. 
So the maximum classes of VCdim $n-2$ are by $2^k-k-1$ larger than the
constructed extremal maximal class of VCdim $n-2$.

%

\end{example}

\subsection{Basic properties of extremal classes}
Extremal classes have a rich combinatorial structure (See~\citep{DBLP:journals/corr/abs-1211-2980} and references within for more details).
We discuss some of parts which are relevant to compression schemes.
         
The following theorem provides alternative
characterizations of extremal classes:
\begin{theorem}[\cite{BR95,Dress2}]\label{thm:extrchar}
The following statements are equivalent:
\vspace{-0.1in}
	\begin{enumerate}
	\item{$C$ is extremal, i.e.  $s(C)=\st(C)$.}
        \item{$|s(C)|=|\st(C)|$.}
	\item{$\lvert \sstr(C)\rvert=\lvert C\rvert$.}
	\item{$\lvert C\rvert=\lvert \str(C)\rvert$.} 
	\item{$\{0,1\}^n\setminus C$ is extremal.}
	\end{enumerate}
\end{theorem}
The following theorem shows that the property of ``being an
extremal class'' is preserved under standard
operations. It was also proven independently by several authors
(e.g.~\citep{BR95,Dress2}).
\begin{theorem}
Let $C$ be any extremal class, $S\subseteq \dom(C)$,
and $B$ be any cube such that $\dom(B)=\dom(C)$. 
Then $C-S$ and $C^S$ are extremal concept classes over the domain
$\dom(C)-S$ and $B\cap C$ is an extremal concept class over
the domain $\dom(C)$.
\label{thm:operations}
\end{theorem}
Note that if $C$ is maximum then $C-S$ and $C^S$ are also
maximum, but $B\cap C$ is not necessarily maximum.
This is an example of the advantage extremal classes have
over the more restricted notion of maximum classes.
         
Interestingly, the fact that extremal classes are preserved under intersecting with cubes yields
a rather simple proof 
(communicated to us by Ami Litman) 
of the fact that every extremal class is ``distance preserving''. 
This property also holds for maximum classes~\citep{Welzl}, 
however the proof for extremal classes is
much simpler than the previous proof for maximum classes (given in~\citep{Welzl}):
\begin{theorem}[\cite{Greco98}]\label{thm:isom}
Let $C$ be any extremal class. Then for every
$c_0,c_1\in C$, the distance between $c_0$ and $c_1$
in the one-inclusion graph of $C$ equals the hamming distance between $c_1$ and $c_2$.
\end{theorem}
\begin{proof}
Assume towards contradiction that this is not the case.
Among all possible pairs of $c_0,c_1\in C$ for which there is no
such path, pick a pair $c_0,c_1$ of a minimal hamming distance.
Let $B$ be the minimal cube over the domain $\dom(C)$ which contains both
$c_0$ and $c_1$. So the dimensions set of $B$ is $\dim(B)=\{x:c_0(x)\neq c_1(x)\}$ and $|\dim(B)|\geq 2$.

We first claim that
by the minimality criteria according to which $c_0,c_1$
were chosen, there cannot be any other concept in $B\cap C$ except $c_0$ 
and $c_1$.
Without loss of generality assume that for all $x\in \dim(B)$, $c_0(x)=0$
and $c_1(x)=1$ (Otherwise we can flip the bits of 
entire columns without affecting the distances between concepts).
If there was now another concept $c\in B$,
then $c|\dim(B)$ must have at least one $0$ and at least one
$1$. By the minimality according to which $c_0,c_1$
were chosen -- there must exist a path between $c_0$
and $c$ of length equal the number of 1's in $c|\dim(B)$.
Similarly, there must exist a path between $c$ and $c_1$
of length equal the number of 0's. The combined path would
be the length of the Hamming distance between $c_0$ and $c_1$. 
So by the minimality $B\cap C$ does not contain another concept.
Therefore $B\cap C=\{c_0,c_1\}$ and this completes the
proof of the claim.

Next we observe that by Theorem \ref{thm:operations}, $B\cap C$ must be extremal.
However we claim that $B\cap C$ is not extremal:
Since there is no edge between its two
concepts $c_0$ and $c_1$, $\st(B\cap C)=\{\emptyset\}$
and therefore $|\st(C)|=1<2=|C|$. This means that
$B\cap C$ is not extremal which is a contradition.
\end{proof}

The following lemma brings out the special cubical structure
of extremal classes. We will use it
to prove the correctness of the compression
scheme given in the following section.
It shows that if $B_1$ and $B_2$ are two maximal cubes of an extremal class $C$ 
then their dimensions sets $\dim(B_1)$ and $\dim(B_2)$
are incomparable.
\begin{lemma}\label{lem:maxcub}
Given $B_1$ and $B_2$ are two cubes of an extremal class $C$. 
If $B_1$ is maximal, then
$$\dim(B_1)\subseteq\dim(B_2)\implies B_1=B_2.$$
\end{lemma}
\begin{proof}
Assume towards contradiction that
$\dim(B_2)\supseteq\dim(B_1):=D$ and $B_2\ne B_1$.
The cube $B_1^{D}$ contains the single concept $\mathrm{tag}(B_1)$ in $C^{D}$, and the cube $B_2^{D}$ is a cube of $C^D$ with dimensions set $\dim(B_2)\setminus D$.
Since $C^{D}$ is extremal, $C^{D}$ must be connected (by Theorem \ref{thm:isom}).
Therefore there is a path in $C^D$ between the concept $\mathrm{tag}(B_1)$ and
some concept in the cube $B_2^{D}$.
This means there is some edge $e$ incident to the concept $B_1^{D}$ in $C^D$. 
This edge $e$ is a one-dimensional cube of $C^D$ labeled
with some dimension $x$. This cube with dimension set
$\dim(e)=\{x\}$ expands to a cube of $C$ with dimension
set $\dim(B_1)\cup\{x\}$ which contains the cube $B_1$ of $C$.
This contradicts the maximality of cube $B_1$ of $C$.
\end{proof}

One concise way to represent an extremal class
is as the union of its maximal cubes. With this
representation, the extremal class of Fig.\ \ref{fig:compr}
is described by the expression
$$*\!*\!0\!*\!00+1\!*\!*\!*\!00+1101\!\!*\!0+01010*,$$ 
where ``$+$'' stands for union.
Note that the dimension sets of the cubes are marked as *'s
and for the class to be extremal, the dimension sets must be
incomparable.

\section{A labeled compression scheme for extremal classes} \label{sec:labeled}

\begin{algorithm}[H]
Let $C$ be an extremal class.
The compression map:
\begin{itemize}
\item
Input: A sample $s$ of $C$.
\item
Output: A subsample $s'=s|\dim(B)$, where $B$ is 
any maximal cube of $C|\dom(s)$ that contains the sample $s$.
\end{itemize}
The reconstruction map:
\begin{itemize}
\item
Input: A sample $s'$ of size at most $\VCdim(C)$.
\item
Output: 
Any concept $h$ which is consistent with $s'$ on $\dom(s')$ and 
belongs to a cube $B$ of $C$ with dimensions set $\dom(s')$.
\end{itemize}
\caption{\label{algo:label} \bf (A labeled compression scheme for extremal classes)}
\end{algorithm}

Let $C$ be a concept class.
On a high level, a sample compression scheme for $C$ compresses 
every sample of $C$ to a subsample of size at most $k$
and this subsample represents a hypothesis on the entire
domain of $C$ that must be consistent with the original sample.
More formally, a labeled compression scheme of size $k$ for $C$
consists of a compression map $\kappa$ and a reconstruction map $\rho$.
The domain of the compression map consists of all samples
from concepts in $C$: 
For each sample $s$, $\kappa$ compresses it to a subsample $s'$
of size at most $k$.
The domain of the reconstruction function $\rho$ is the set of all samples
of $C$ of size at most $k$.
Each such sample is used by $\rho$ to reconstruct a concept
$h$ with $\dom(h)=\dom(C)$. 
The sample compression scheme must satisfy that for all
samples $s$ of $C$,
$$\rho(\kappa(s))\:|\dom(s) = s.$$
The sample compression scheme is said to be proper if the reconstructed hypothesis $h$
always belongs to the original concept class $C$. 

A proper labeled compression scheme for extremal classes 
of size at most the VC dimension 
is given in Algorithm~\ref{algo:label}.
Let $C$ be an extremal concept class and
$s$ be a sample of $C$.
In the compression phase the algorithm finds any {\em maximal cube} $B$
of $C|\dom(s)$ that contains the sample $s$ 
and compresses $s$ to the subsample determined by the dimensions set of that maximal cube.
Note that the size of the dimension set (and the
compression scheme) is bounded by the VC dimension.

How should we reconstruct?
Consider all concepts of $C$ that are consistent with the sample $s$:
$$H_s=\{h\in C: h|\dom(s)=s\}.$$
Correctness means that we need to reconstruct to one of those concepts.
Let $s'$ be the input for the reconstruction function and let $D:=\dom(s')$.
During the reconstruction, the domain $\dom(s)$ of the original sample $s$ is not known.
All that is known at this point is that
$D$ is the dimensions set of a maximal cube $B$ of $C\proj
\dom(s)$
that contained the sample $s$. 
The reconstruction map of the algorithm outputs a concept in the following set:
\begin{align*}
H_B:=&\{h\in C: h \text{ lies in cube } B' \text{such that}\\
&\dim(B')=\dim(B) \text{ and }h|\dim(B)=s|\dim(B)\}.
\end{align*}
For the correctness of the compression scheme
it suffices to show that for all choices of the maximal cube $B$ of $C\proj\dom(s)$, 
$H_B$ is non-empty and a subset of $H_s$.
The following Lemma guarantees the non-emptiness.
\begin{lemma}\label{lem:crucial1}
Let $C$ be an extremal class and let $D\subseteq\dom(C)$
be the dimensions set of some cube of $C\proj \dom(s)$.
Then $D$ is also the dimensions set of some cube of $C$.
\end{lemma}
\begin{proof}
Clearly the dimension set $D$ is shattered by $C\proj
\dom(s)$ 
and therefore it is also shattered by $C$. By the extremality of $C$,
$D$ is also strongly shattered by it, and thus
there exists a cube $B$ of $C$ with dimensions set $D$.
\end{proof}
The second lemma show that for each choice of the maximal
cube $B$, $H_B\subseteq H_s$.
\begin{lemma}\label{lem:crucial2}
Let $s$ be a sample of an extremal class $C$,
let $B$ be any maximal cube of $C|\dom(s)$ that contains
$s$, and let $D$ denote the dimensions set of $B$.
Then for any cube $B'$ of $C$ with $\dim(B')=D$, the concept $h\in B'$ that is consistent
with $s$ on $D$ is also consistent with $s$ on $\dom(s)\setminus D$.
\end{lemma}

\begin{figure}
\label{fig:compr}
\begin{tabular}{cc}
\hspace{-2.5cm}
\begin{minipage}[h]{.8\textwidth}
\includegraphics[width=\textwidth,angle=0]{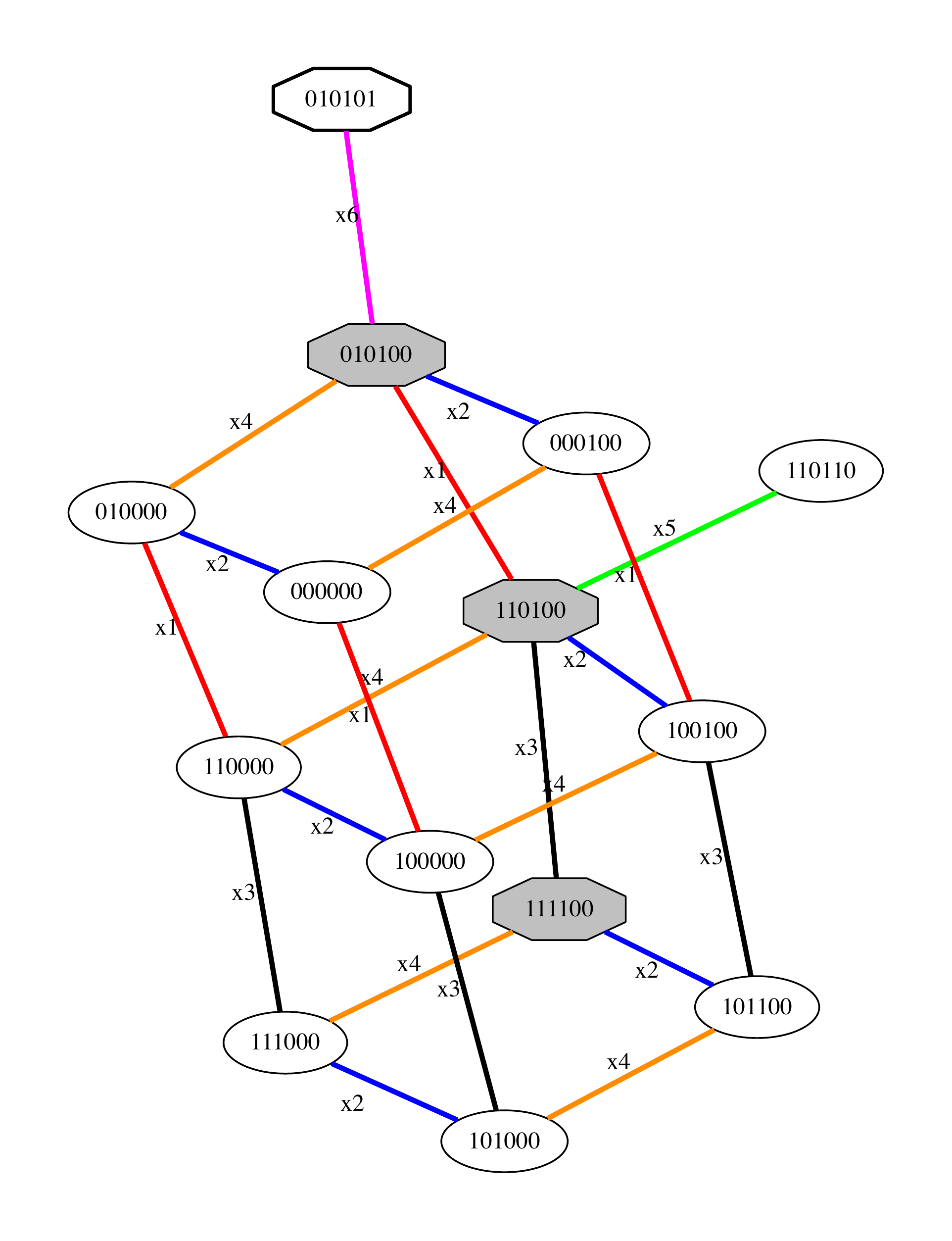}
\end{minipage}
&
\begin{minipage}[h]{.6\textwidth}
\hspace{-4cm}
\vspace*{-1cm}
\includegraphics[angle=+2,width=1.2\textwidth]{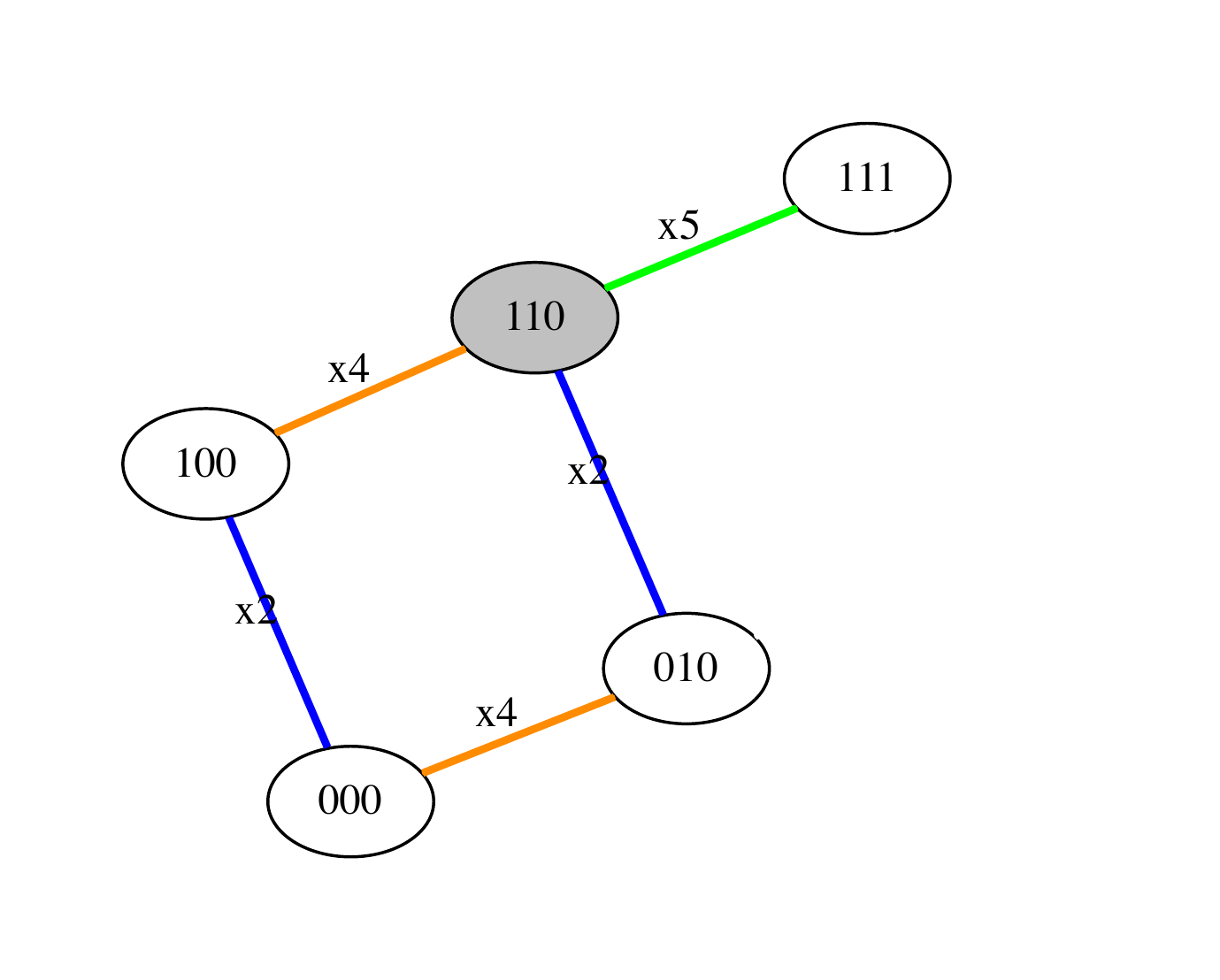}
\end{minipage}
\end{tabular}
\caption{
The one-inclusion graph of an extremal concept class 
$C$ is given on the left. 
Consider the sample
$s=\overset{x_2}{\one}\overset{x_4}{\one}\overset{x_5}{\zero}$.
There are 4 concepts $c\in C$ consistent
with this sample (the octagonal vertices), i.e.
$H_s=\{
1\one1\one\zero0,
1\one0\one\zero0,
0\one0\one\zero0,
0\one0\one\zero1
\}$.
There are 2 maximal cubes of $C|\dom(s)$ 
(graph on right) that contain the
sample $s$ (in grey) with dimension sets $\{x_5\}$ and $\{x_2,x_4\}$, respectively.
Let $B$ be the maximal cube with dimension set $D= \{x_2,x_4\}$.
There are 3 cubes of $C$ (on left) with the same dimension set $D$.
Each contains a concept $h$ (shaded grey) that is consistent with the original
sample on $D$, i.e.
$h|D=s|D=\overset{x_2}{\one}\overset{x_4}{\one}$ 
and therefore
$H_B=\{1\one1\one00,
       1\one0\one00,
       0\one0\one00
\}$. 
For the correctness we need that $H_B$ (grey nodes on left) 
is non-empty and a subset of $H_s$ (octagon nodes on left). 
Note that in this case
$H_B$ is a strict subset.
}
\end{figure}

\begin{proof}
Since $B$ is a cube with dimensions set $D$, 
$B|(\dom(s)\setminus D)$ contains the single concept
${\mathrm tag}(B)$.

Let $B'$ be any cube of $C$ with $\dim(B')=D$, and let $h$ be the concept in $B'$
which is consistent with $s$ on $D$.
Now consider the cube $B' \proj \dom(s)$.
We will show that $B'\proj \dom(s)=B$. This will finish
the proof as it shows that both $h\proj \dom(s)$ and $s$ 
belong to $B'\proj \dom(s) = B$ 
which means that ${\mathrm tag}(B)=h\proj(\dom(s)\setminus D)=s\proj(\dom(s)\setminus D)$.
Moreover, by the definition of $h$, $h\proj D = s\proj D$, and therefore $h\proj \dom(s) = s$ as required.

We now show that $B'\proj \dom(s) = B$.
Indeed, since $B'$ is a cube of $C$
with dimension set $D \subseteq \dom(s)$, the cube
$B'|\dom(s)$ is a
cube of $C|\dom(s)$ with the same dimension set $D$.
Thus the dimension set of $B'|\dom(s)$ contains the dimension
set of the maximal cube $B$ of $C|\dom(s)$.
Therefore, since $C\proj \dom(s)$ is extremal (Theorem~\ref{thm:operations})
it follows by Lemma~\ref{lem:maxcub} that $B'|\dom(s)=B$.
\end{proof}

\section{Unlabeled sample compression schemes and related combinatorial conjectures}\label{sec:unlabeled}
\begin{figure}[t]
\large
\begin{center}
\setlength{\tabcolsep}{2pt}
\begin{tabular}{c @{\hskip 1.5in} c}
  \begin{tabular}{cccccc}
  $x_1$&$x_2$&$x_3$&$x_4$&$x_5$&$x_6$\\
\hline\hline
  0&1&0&           1&0&\ul{1}\\
  \ul{0}&\ul{1}&0& \ul{1}&0&0\\
  \ul{0}&0&0&      \ul{1}&0&0\\
  1&1&0&           1&\ul{1}&0\\
  \ul{0}&\ul{1}&0& 0&0&0\\
  \ul{0}&0&0&      0&0&0\\
  1&\ul{1}&\ul{0}& \ul{1}&0&0\\
  1&0&\ul{0}&      \ul{1}&0&0\\
  1&\ul{1}&\ul{0}& 0&0&0\\
  1&0&\ul{0}&      0&0&0\\
  1&\ul{1}&1&      \ul{1}&0&0\\
  1&0&1&           \ul{1}&0&0\\
  1&\ul{1}&1&      0&0&0\\
  1&0&1&           0&0&0
  \end{tabular}
&
  \begin{tabular}{c 
                  >{\columncolor{Gray}}c
                  c
                  >{\columncolor{Gray}}c
                  >{\columncolor{Gray}}c
                  c
                  c}
  $x_1$&$x_2$&$x_3$&$x_4$&$x_5$&$x_6$&\\
  \hline \hline
  \ul{0}&\ul{1}&0& \ul{1}&0&0&\\
  0&1&0&           1&0&\ul{1}&\\
  1&\ul{1}&\ul{0}& \ul{1}&0&0&\\
  1&\ul{1}&1&      \ul{1}&0&0&$\boldsymbol{\leftarrow}$\\
\hline
  \ul{0}&0&0&      0&0&0&\\
  1&0&\ul{0}&      0&0&0&\\
  1&0&1&           0&0&0&$\boldsymbol{\leftarrow}$\\
\hline
  \ul{0}&0&0&      \ul{1}&0&0&\\
  1&0&\ul{0}&      \ul{1}&0&0&\\
  1&0&1&           \ul{1}&0&0&$\boldsymbol{\leftarrow}$\\
\hline
  \ul{0}&\ul{1}&0& 0&0&0&\\
  1&\ul{1}&\ul{0}& 0&0&0&\\
  1&\ul{1}&1&      0&0&0&$\boldsymbol{\leftarrow}$\\
\hline
  1&1&0&           1&\ul{1}&0&$\boldsymbol{\leftarrow}$\\
  \end{tabular}
\end{tabular}
\end{center}
\caption{\label{fig:unlab}
Unlabeled compression scheme based on peeling.
The vertices of the extremal class $C$ from the left figure of Fig.\ \ref{fig:compr}
were peeled in a top down order (See table on left). 
Note that the highest vertex is always a corner of the remaining extremal class below. 
The resulting representation sets are underlined in the left table.
For example the second concept is represented by the set $\{x_1,x_2,x_4\}$ 
and the last to the empty set.
\newline
Now consider the domain $D=\{x_2,x_4,x_5\}$ 
and a sample $s= \overset{x_2}{1}\overset{x_4}{1}\overset{x_5}{0}$ over this domain.
Partition $C$ into equivalence classes such that concepts in
the same class are consistent on $D$.
In the table on the right, we reordered and segmented the
concepts of $C$ by their equivalence classes.
Each class corresponds to a member of $C|D$.
In Lemma \ref{lemma:clash} we show that
each class contains exactly one concept $c$ such that
$r(c)\subseteq D$ (marked with $\boldsymbol{\leftarrow})$.
Each sample $s$ of $C|D$ is compressed to the unique subset of $D$
in the equivalence class that represents this consistent concept  
(marked with $\boldsymbol{\leftarrow}$). 
In the reconstruction, each representation set is reconstructed to the concept
it represents.
In particular the sample $s$ associated with the first class is compressed to
$\{x_2,x_4\}\subseteq D$ which represents the
consistent concept $111110$ and ``unlabeled sub sample'' $\{x_2,x_4\}$ 
is reconstructed to this concept.
}
\end{figure}

The labeled compression scheme of the previous section
compresses each sample of the concept class to a (labeled)
subsample and this subsample is guaranteed to represent a hypothesis
that is consistent with the entire original sample.
Such a labeled compression scheme (of size equal the VC
dimension $d$) was first found for maximum classes.
In the previous section, we generalized this scheme to extremal classes.

Alternate ``unlabeled'' compression schemes have also
been found for maximum classes and a natural question is whether
these schemes again generalize to extremal classes.
As we shall see there is an excellent match between
the combinatorics of unlabeled compression schemes and
extremal classes. The existence of such schemes remains open
at this point. We can however relate their existence 
to some natural conjectures about extremal classes. 

An unlabeled compression schemes compresses a sample $s$ of the
concept class $C$ to an (unlabeled) subset of the domain of the sample $s$.
In other words, in an unlabeled compression scheme the labels of the original
sample are not used by the reconstruction map. 
The size of the compression scheme
is now the maximum size of the subset that the sample is 
compressed to.
Consider an unlabeled compression scheme for $C$
of size $\VCdim(C)$.
For a moment restrict your attention
to samples of $C$ over some fixed domain $S \subseteq \dom(C)$.
Each such sample is a concept in the restriction $C|S$. 
Note that two different concepts in $C\proj S$ must be compressed
to different subsets of $S$, otherwise if they were compressed
to the same subset, the reconstruction of it would not be consistent with one 
of them. For maximum classes,
the number of concepts in $C|S$ is exactly the number of
subsets of $S$ of size up to the VC dimension.
Intuitively, this ``tightness'' makes 
unlabeled compression schemes combinatorially rich and interesting.

Previous unlabeled compression schemes for maximum classes were based on 
``representation maps''.
For maximum classes these are one-to-one mappings between $C$ and subsets of $\dom(C)$ of size at most $\VCdim(C)$. Representation maps were used in the following way: Each sample $s$ is compressed to a subset of $\dom(s)$ which represents a consistent hypothesis with $s$, and each subset of size at most $\VCdim(C)$ of $\dom(C)$ is reconstructed to the hypothesis it represents.
Clearly, not every one-to-one mapping between $C$ and subsets of $\dom(C)$ of size at most $\VCdim(C)$ yields an unlabeled compression scheme in this manner, and finding a good representation map (or proving that one exists) became the focus of many previous works.

For maximum classes, representation
maps $r$ have been found that map (one-to-one) the concept class $C$
to all subsets of $\dom(C)$ of size up to the VC dimension of $C$.
The key combinatorial property that enabled finding representation maps
for maximum classes was a ``non clashing'' condition~\citep{DBLP:journals/jmlr/KuzminW07}.
This property was used to show that 
for any sample $s$ of $C$ there is exactly one concept $c$ that
is consistent with $s$ and $r(c)\subseteq \dom(s)$.
This immediately implies an unlabeled compression scheme based on non clashing
representation maps: Compress to the
unique subset of the domain of the sample that represents
a concept consistent with the given sample.

%

We will show below that representation maps naturally
generalize to extremal classes: $r$ must now map
(one-to-one) the extremal class $C$ to its 
shattered sets $\s(C)$.
This is natural since for extremal classes $|C|=|\s(C)|$.
We will see that again, if the non clashing condition holds, 
then for any sample $s$ of $C$ there is exactly one concept $c$ that
is consistent with $s$ and $r(c)\subseteq \dom(s)$.

For maximum classes, such representation maps were first shown to
exist via a recursive construction \citep{DBLP:journals/jmlr/KuzminW07}.
Alternate representation maps were also proposed in
\citep{DBLP:journals/jmlr/KuzminW07} based on a certain greedy ``peeling''
algorithm that iteratively assigns a representation
to a concept and removes this concept from
the class. The correctness of the representation maps based on
peeling was finally established in \citep{RR3}.
In this section, we show that existence of representation
maps based on peeling hinges on certain natural and concise properties
of extremal classes. However establishing these conjectured properties 
of extremal classes remains open.

\paragraph{Representation maps.}
For any concept class $C$ a {\em representation map} is any
one-to-one mapping from concepts to subsets of the domain, i.e.
$r:C\rightarrow\powset(\dom(C))$. 
We say that $c\in C$ {\em is represented} by the 
{\em representation set} $r(c)$.
Furthermore we say that two different concepts $c,c'$ {\em clash}
with respect to $r$ if they are consistent with each other on the union of their
representation sets, i.e. 
$c|\left(r(c)\cup r(c')\right) = c'|\left(r(c)\cup r(c')\right)$. 
If no two concepts clash then we say that $r$ is \emph{non clashing}.

\begin{example}
{\bf (Non clashing maps based on disagreements)}
For an arbitrary concept class $C$ and $c_1,c_2\in C$, let $dis(c_1,c_2)$
be the set of all dimensions on which $c_1,c_2$ disagree,
i.e.  $dis(c_1,c_2) =\{x\in \dom(C)~:~c_1(x)\neq c_2(x)\}$.
Now let $c_0:\dom(C)\rightarrow\{0,1\}$ be a fixed ``reference'' concept
and define a representation map for class $C$ as $r(c):=dis(c,c_0)$. 
We leave it to the reader to verify that $r$ is non clashing.
\end{example}

\begin{example}
{\bf (A Non clashing representation map for distance preserving classes)}
Let $C$ be a distance preserving class, that is
for every $u,v\in C$, the distance between $u,v$ in the
one-inclusion graph of $C$ equals to their hamming distance.
For every $c\in C$, define 
$$deg_C(c)=\{x\in \dom(C)~:~c\mbox{ is incident to an }x\mbox{-edge in the one-inclusion graph of }C\}.$$
The representation map $r(c):=deg_C(c)$ has the property that for every $c\neq c'\in C$,
$c$ and $c'$ disagree on $r(c)$. To see this, note that
since $C$ is isometric then any shortest path from
$c$ to $c'$ in $C$ traverses exactly the
dimensions on which $c$ and $c'$ disagrees.
In particular, the first edge leaving $c$
in this path traverses a dimension $x$ for 
which $c(x)\neq c'(x)$. By the definition
of $deg_C(c)$ we have that $x\in deg_C(c)$ and 
indeed $c$ and $c'$ disagree on $deg_C(c)$.
\end{example}

In fact, this gives a stronger 
property for distance preserving classes, 
which is summarized in the following lemma.
This lemma will be useful in our analysis.
\begin{lemma}\label{lem:teaching}
Let $C$ be a distance preserving class and let $c\in C$.
Then $deg_C(c)$ is a teaching set for $c$
with respect to $C$. That is, for all $c'\in C$:
$$c'\neq c\implies \exists x\in deg_C(c):~c(x)\neq c'(x).$$
\end{lemma}
Clearly the representation map $r(c)=deg_C(c)$ is non clashing.
The following lemma establishes that certain non clashing
representation maps immediately give unlabeled compression schemes:
\begin{lemma}\label{lemma:clash}
Let $r$ be any representation map that is a bijection
between an extremal class $C$ and $\sstr(C)$.
Then the following two statements are equivalent:
\begin{enumerate}
\item
$r$ is non clashing.
\item
For every sample $s$ of $C$, there is exactly one concept
$c\in C$ that is consistent with $s$ and $r(c)\subseteq \dom(s)$.
\end{enumerate}
\end{lemma}

Based on this lemma it is easy to see
that a representation mapping $r$ for
an extremal concept class $C$ defines a compression
scheme as follows (See Algorithm \ref{algo:unlabel}
and an example in Fig.\ \ref{fig:unlab}).
For any sample $s$ of $C$ we {\em compress} $s$ to
the unique representative $r(c)$ such that
$c$ is consistent with $s$ and $r(c)\subseteq \dom(s)$.
Reconstruction is even simpler, since $r$ is
bijective: If $s$ is compressed to the set $r(c)$,
then we reconstruct $r(c)$ to the concept $c$. 

Note that the representation set $r(c)$ of a concept $c$
is always an unlabeled set from $\ss(C)$.
However, we could also compress to the labeled subsamples $c|r(c)$.
It is just that the labels in this type of scheme do not
have any additional information and are redundant.

\begin{algorithm}[H]~

The compression map.

 Input: A sample $s$ of $C$.
\begin{enumerate}
\item Let $c\in C$ be the unique concept which satsifies
(i) $c|\dom(s)=s$,
and (ii) $r(c)\subseteq \dom(s)$ 
\item Output $r(c)$.
\end{enumerate}

The reconstruction map.

Input: a set $S'\in\sstr(C)$
\begin{enumerate}
\item Since $r$ is a bijection between $C$ and $\sstr(C)$, there is a unique
$c$ such that $r(c)=S'$.
\item Output $c$.
\end{enumerate}
\caption{\label{algo:unlabel} \bf (An unlabeled compression scheme from a representation map)}
\end{algorithm}

\begin{proof}[\it of Lemma~\ref{lemma:clash}]

$2 \Rightarrow 1:$
Proof by contrapositive. Assume $\neg 1$, that is:
$\exists c, c'\in C,\, c\neq c' \text{ such that } 
c|r(c)\cup r(c') = c'|r(c)\cup r(c')$. 
Then let $s = c|r(c) \cup r(c')$.
Clearly both $c$ and $c'$ are consistent with $s$
and $r(c),r(c') \subseteq \dom(s)$.
This negates 2.

$1 \Rightarrow 2:$ 
{
We will show that 1 implies the following equivalent form
of 2: For all sample domains $D\subseteq \dom(C)$ and samples $s\in C|D$,
there is exactly one concept $c\in C$ that is consistent
with $s$ and $r(c) \subseteq D$.
Recall that any domain $D\subseteq \dom(C)$ partitions $C$
into equivalence classes where
each class contains all concepts of $C$ consistent with a sample from $C|D$.
We need to show that each equivalence class has a unique concept in
$R:=\{c:r(c)\in D\}$. 
See Fig.\ \ref{fig:unlab} for an example.
We split our goal into two parts: 
\vspace*{-.1cm}
\begin{enumerate}
\item[(a)] $C\proj D = R\proj D$, 
i.e. for every $s\in C \proj D$ there 
is at least one $c\in R$ such that $s=c\proj D$ and
\item[(b)] $|R|D| = |R|,$ 
i.e.  for each sample $s'\in R|D$
there is at most $c\in R$ such that $s'=c|D$.
\end{enumerate}

We first prove Part (b).
Clearly $|R|D| \le |R|.$
Furthermore, the non-clashing condition (Part 1 of the lemma) 
implies that any distinct concepts $c_1,c_2\in R$ disagree on
$r(c_1)\cup r(c_2)\subseteq D$ and therefore $|R\proj D| = |R|.$

Since $R|D \subseteq C|D$, 
the set equality $R|D= C|D$ of Part (a) is implied by the fact that
both sets have the same cardinality:
\begin{align*}
\lvert C\proj D \rvert &= \lvert\str(C\proj D)\rvert 
\tag{since $C\proj D$ is extremal}\\
                &= \lvert\str(C)\cap\mathcal{P}(D)\rvert
\tag{holds for every concept class $C$ and $D\subseteq\dom(C)$}\\
                &=\lvert R\rvert\tag{since $r:C\rightarrow\str(C)$ is a bijection}\\
                &= \lvert R\proj D\rvert\tag{by Part
(b).}
\end{align*}
}
\end{proof}

For a more detailed proof Assume $\neg 2$, i.e.
there is a sample $y$ of $C$ with $\dom(y)=Y$ for which
there are either zero or (at least) two
consistent concepts $c$ for which
$r(c) \subseteq Y$.
If two concepts $c, c'\in C$ are consistent with $y$
and $r(c),r(c')\subseteq Y$,
then $c|r(c)\cup r(c') = c'|r(c)\cup r(c')$
(which is $\neg 1$).
Assume now that there is no concept $c$ consistent with 
some sample $y$ of $C$ for which $r(c) \subseteq Y$. 
Note that
\begin{align*}
|C\proj Y| &= \lvert\sstr(C\proj Y)\rvert
\tag{Since $C\proj Y$ is extremal.}\\
                &= \lvert\sstr(C)\cap\mathcal{P}(Y)\rvert\\
                &=\lvert\{c:r(c)\subseteq Y\}\rvert\tag{Since $r:C\rightarrow\sstr(C)$ is a bijection.}
\end{align*}
In other words the number of samples consistent with $y$
equals the number of concepts with a representation set in $Y$.
Partition $C$ into equivalence classes where two concepts
$c,c'$ are equivalent if $c|Y=c'|Y$ (See Fig.\
\ref{fig:unlab} for a running example).
Thus, each equivalence class corresponds to a sample of $C$
with domain $Y$. Each concept is identified by its
representation set $r(c)$ and the number of equivalence
classes equals $\lvert\{c:r(c)\subseteq Y\}\rvert$.
By assumption, all concepts $c$ in the equivalence class of
sample $y$ have $r(c)\not\subseteq Y$. Therefore
by a pigeon hole argument there must be an
equivalence class with two distinct concepts 
$c_1,c_2\in C$ for which
$r(c_1),r(c_2)\subseteq Y.$
These two concept clash and again $\neg 1$ is implied.{\hspace*{\fill}$\Box$\par}

Once we have a valid representation mapping for some extremal concept class $C$, 
we can easily derive a valid mapping for any restriction of the class $C|A$ by compressing every 
restricted concept. This is discussed in the following corollary.

\begin{corollary} For any extremal class $C$ and $A \subseteq \dom(C)$, if $r$ is a representation mapping for $C$ then
a representation mapping for $C|A$ can be constructed as follows. For any $c \in C|A$, let
$r_A(c)$ be the representative of the unique 
concept $c' \in C$, such that $c'|A = c$ and $r(c') \subseteq A$.
\label{cor_restrict}
\end{corollary}
\begin{proof}
The construction of the mapping for $C|A$ essentially tells us to treat the concept $c$ as a sample from $C$
and to compress it. Thus we can apply Lemma
\ref{lemma:clash} to see that $r_A(c) \subseteq A$ is always uniquely 
defined. Now we need to show that $r_A$ 
satisfies the conditions of the Main Definition. 
Since the representatives $r_A(c)$ are subsets of $A$,
the non-clashing property for the representation mapping
$r_A$ for $C|A$
follows from the non-clashing condition for $r$ for $C$.
The bijection property follows from a counting
argument like the one used in the proof of Lemma 
\ref{lemma:clash}, since 
$\size(C|A) = \size(\{r(c) \text{ s.t. } r(c) \subseteq A\})$.
\end{proof}

\paragraph{Corner peeling yields good representation maps.}
We now present a natural conjecture concerning extremal classes
and show how this conjecture can be used to construct non clashing representation maps.
A concept $c$ of an extremal class $C$ is a {\em corner} of $C$ if
$C\setminus\{c\}$ is extremal.
By Lemma~\ref{lem:maxcub} we have that for
each $S\subseteq \dom(C)$ there is at most one
maximal cube with dimension set $S$ and if $S$
is the dimensions set of a non-maximal cube, then
there are at least two cubes with this dimension set. Therefore
$$\ss(C\setminus\{c\})
= \ss(C) \setminus 
\{\dim(B): B \text{ is maximal cube of $C$ containing $c$}\}.$$ 
For $C\setminus\{c\}$ to be extremal,
$|\ss(C\setminus\{c\})|$ must be $|C|-1$ (by Theorem \ref{thm:extrchar})
and therefore $c$ is a corner of an extremal class $C$ iff
$c$ lies in exactly one maximal cube of $C$.

\begin{conjecture}\label{conj:corner-peeling}
Every non empty extremal class $C$ has at least one corner.
\end{conjecture}
In \citep{DBLP:journals/jmlr/KuzminW07}
essentially the same conjecture was presented for maximum classes.
For these latter classes, the conjecture was finally proved in \citep{RR3}.
This conjecture also has been proven for other special
cases such as extremal classes of VC dimension at most $2$
\citep{Litman12,MeszarosR14}. In fact \cite{Litman12} proved a stronger statement: 
For every two extremal classes $C_1\subseteq C_2$ such that $\VCdim(C_2)\leq 2$ 
and $|C_2\setminus C_1|\geq 2$, there exists an extremal class $C$ such that $C_1\subset C\subset C_2$ (i.e. $C$ is a strict subset of $C_2$ and a strict superset of $C_1$).
Indeed, this statement is stronger as by repeatedly picking
a larger extremal class $C_1\subseteq C_2$ eventually a $c\in C_2$ is obtained such that $C_2-\{c\}$ is extremal.
For general extremal classes this stronger statement also remains open.
\begin{conjecture}\label{conj:intermediate}
For every two extremal classes $C_1\subseteq C_2$ with $|C_2\setminus C_1|\geq 2$
there exists an extremal class $C$ such that $C_1\subset C\subset C_2$.
\end{conjecture}

Let us return to the more basic
Conjecture~\ref{conj:corner-peeling}.
How does this conjecture yield a representation map?
Define an order\footnote{Such orderings are related to the recursive teaching dimension which was studied by~\cite{DoliwaSZ10}} on $C$ 
$$c_1,c_2\ldots c_{|C|}$$
such that for every $i$,
$c_i$ is a corner of $C_i=\{c_j : j\geq i\}$, and define a map $r:C\rightarrow\sstr(C)$
such that $r(c_i)=\dim(B_i)$ where $B_i$ is the unique maximal cube of
$C_i$ that $c_i$ belongs to. We claim that $r$ is a representation map. Indeed, $r$ is a
one-to-one mapping from $C$ to $\sstr(C)$ (and since $C$ is extremal $r$ is a bijection).
To see that $r$ is non clashing, note that
$r(c_i)=\dim(B_i)=\deg_{C_i}(c_i)$. $C_i$ is extremal and therefore distance preserving (Theorem~\ref{thm:isom}). Thus, Lemma~\ref{lem:teaching} implies that $r(c_i)$ is a teaching set of $c_i$ with respect to $C_i$. This implies that $r$ is indeed non clashing.

\section{Discussion}
We studied the conjecture of \cite{DBLP:journals/ml/FloydW95} 
which asserts that every concept classes has a sample compression scheme 
of size linear in its VC dimension.
We extended the family of concept classes for which the conjecture is known to 
hold by showing that every extremal class has a sample compression scheme of size equal to its VC dimension. We discussed the fact that extremal classes form a natural 
and rich generalization of maximum classes for which the
conjecture had been proved before~\citep{DBLP:journals/ml/FloydW95}.

We further related basic conjectures concerning the
combinatorial structure of extremal classes with the
existence of optimal unlabeled compression schemes. These
connections may also be used in the future to provide a
better understanding on the combinatorial structure of
extremal classes, which is considered to be incomplete by several authors~\citep{BR95,Greco98,S-ext}. 

Our compression schemes for extremal classes yield another
direction of attacking the general conjecture of Floyd and
Warmuth: it is enough to show that an arbitrary maximal concept class of 
VC dimension $d$ can be 
covered by $\exp(d)$ extremal classes of VC dimension $O(d)$.
Note it takes additional $O(d)$ bits to specify which of the
$\exp(d)$ extremal classes is used in the compression.

\paragraph{\bf Acknowledgements} We thank Micha\l$\;$Derezi\'nski
for a good feedback on the writing of the paper and Ami
Litman for helpful combinatorial insights.


\bibliography{compRef}

\appendix

\section{Proof of Claim~\ref{lem:maximalextremal}}\label{app:maximalextremal}
To prove this claim use the following simple fact.
\begin{lemma}[\cite{DBLP:journals/corr/abs-1211-2980,ShatNews,BR95}]
Let $C,\bC$ be two complementing concept classes over domain $X$.
Then for every $Y\subseteq X$ exactly one of the following holds.
\begin{enumerate}
\item $C$ strongly shatters $Y$.
\item $\bC$  shatters $\bar{Y}$.
\end{enumerate}
\end{lemma}
With this lemma at hand, note that if $C$ is {\bf extremal} then for every $Y\subseteq X$, either
$C$ {\bf strongly shatters} $Y$ or $\bC$ {\bf strongly shatters} $\bar{Y}$.

Going back to our $C$ from the construction, it is easy to
verify that $\bC$ (and therefore $C$) is extremal, because glueing an edge of a new dimension to a concept of an extremal class preserves extremality.
 Thus, by the above lemma 
$\VCdim(C)=d=n-2$ (because every subset of size $1$ is strongly shattered by $\bC$ but there are subsets of size $2$ that are not strongly shattered by $\bC$).
To see why $C$ is maximal we again use the above lemma and
observe that every concept $\bar{c}$ which is removed from $\bC$
removes a set of size $1$ (the set containing the unique
dimension of the edge glued to $\bar{c}$) from the strongly shattered sets
of $\bC$. This means that a set of size
$n-1$ is added to the shattered sets of $C$ and
the VC dimension of $C$ is increased from $n-2$ to $n-1$.

\end{document}